\documentclass[reqno]{amsart}

\usepackage[margin=1.4in]{geometry}
\numberwithin{equation}{section}

\usepackage{microtype}
\usepackage{mdwlist}
\usepackage{url}
\usepackage{graphicx,graphics,enumitem,algorithm,algorithmic,amsmath,amssymb,amsfonts,amsthm} 
\usepackage[skip=2pt]{caption}

\newcommand*{\cldots}{\hspace{-2pt}  \mathinner{{\cdotp}{\cdotp}{\cdotp}}\hspace{-1pt}}
\newcommand{\BigO}[1]{\ensuremath{\operatorname{O}\hspace{-1pt}\left(#1\right)}}

\newtheorem{theorem}{Theorem}[section]
\newtheorem{lemma}[theorem]{Lemma}
\theoremstyle{definition}

\theoremstyle{remark}

\begin{document}

\title[Asymptotically Exact, Embarrassingly Parallel MCMC]
        {Asymptotically Exact, Embarrassingly Parallel MCMC}

\author[Willie Neiswanger]{Willie Neiswanger$^1$}
\address{}
\curraddr{}
\email{}
\thanks{}

\author[Chong Wang]{Chong Wang$^2$}
\address{}
\email{}
\thanks{}

\author[Eric Xing]{Eric Xing$^3$}
\address{}
\email{}


\thanks{$^{1}$\texttt{willie@cs.cmu.edu}, $^{2}$\texttt{chongw@cs.princeton.edu},
    $^{3}$\texttt{epxing@cs.cmu.edu}. Last revised \today.
}

\maketitle

\begin{abstract}
Communication costs, resulting from synchronization requirements during
learning, can greatly slow down many parallel machine learning algorithms.  In
this paper, we present a parallel Markov chain Monte Carlo (MCMC) algorithm in
which subsets of data are processed independently, with very little
communication.  First, we arbitrarily partition data onto multiple machines.
Then, on each machine, any classical MCMC method (e.g., Gibbs sampling) may be
used to draw samples from a posterior distribution given the data subset.
Finally, the samples from each machine are combined to form samples from the
full posterior.  This embarrassingly parallel algorithm allows each machine to
act independently on a subset of the data (without communication) until the
final combination stage.  We prove that our algorithm generates asymptotically
exact samples and empirically demonstrate its ability to parallelize burn-in
and sampling in several models.
\end{abstract}

\section{Introduction}
Markov chain Monte Carlo (MCMC) methods are popular tools for performing
approximate Bayesian inference via posterior sampling.  One major benefit of
these techniques is that they guarantee asymptotically exact recovery of the
posterior distribution as the number of posterior samples grows.  However, MCMC
methods may take a prohibitively long time, since for $N$ data points, most
methods must perform $\BigO{N}$ operations to draw a sample.  Furthermore, MCMC
methods might require a large number of ``burn-in'' steps before beginning to
generate representative samples. Further complicating matters is the issue
that, for many big data applications, it is necessary to store and process data
on multiple machines, and so MCMC must be adapted to run in these
data-distributed settings.

Researchers currently tackle these problems independently, in two primary ways.
To speed up sampling, multiple independent chains of MCMC can be run in parallel
\cite{wilkinson2006parallel,laskey2003population,murray2010distributed};
however, each chain is still run on the entire dataset, and there is no
speed-up of the burn-in process (as each chain must still complete the full
burn-in before generating samples). To run MCMC when data is partitioned among
multiple machines, each machine can perform computation that involves a subset
of the data and exchange information at each iteration to draw a sample
\cite{Langford:2009,Newman:2009,smola2010architecture}; however, this requires
a significant amount of communication between machines, which can greatly
increase computation time when machines wait for external information
\cite{Agarwal:2012,Ho:2013}.

We aim to develop a procedure to tackle both problems simultaneously, to allow
for quicker burn-in and sampling in settings where data are partitioned among
machines.
To accomplish this, we propose the following: on each machine, run MCMC on only
a subset of the data (independently, without communication between machines),
and then combine the samples from each machine to algorithmically construct
samples from the full-data posterior distribution.
%
We'd like our procedure to satisfy the following four criteria:
\begin{enumerate}[]
    \item Each machine only has access to a portion of the data.
    \item Each machine performs MCMC independently, without communicating (i.e.
    the procedure is ``embarrassingly parallel'').
    \item Each machine can use any type of MCMC to generate samples.
    \item The combination procedure yields provably asymptotically exact
    samples from the full-data posterior.
\end{enumerate}

The third criterion allows existing MCMC algorithms or software packages to be
run directly on subsets of the data---the combination procedure then acts as a
post-processing step to transform the samples to the correct distribution.
Note that this procedure is particularly suitable for use in a MapReduce
\cite{dean2008mapreduce} framework.  Also note that, unlike current strategies,
this procedure does not involve multiple ``duplicate'' chains (as each chain
uses a different portion of the data and samples from a different posterior
distribution), nor does it involve parallelizing a single chain (as there are
multiple chains operating independently). We will show how this allows our
method to, in fact, parallelize and greatly reduce the time required for
burn-in.

In this paper we will (1) introduce and define the \emph{subposterior}
density---a modified posterior given a subset of the data---which will be used
heavily, (2) present methods for the embarrassingly parallel MCMC and
combination procedure, (3) prove theoretical guarantees about the samples
generated from our algorithm, (4) describe the current scope of the presented
method (i.e. where and when it can be applied), and (5) show empirical results
demonstrating that we can achieve speed-ups for burn-in and sampling while
meeting the above four criteria.

\section{Embarrassingly Parallel MCMC}

The basic idea behind our method is to partition a set of $N$ i.i.d. data
points $x^N = \{x_1, \cdots,x_N\}$ into $M$ subsets, sample from the
\emph{subposterior}---the posterior given a data subset with an underweighted
prior---in parallel, and then combine the resulting samples to form
samples from the full-data posterior $p(\theta | x^N)$, where
$\theta~\in~\mathbb{R}^d$ and $p(\theta | x^N) \propto p(\theta) p(x^N |
\theta) = p(\theta) \prod_{i=1}^N p(x_i | \theta)$.

More formally, given data $x^N$ partitioned into $M$ subsets 
$\{ x^{n_1},\ldots,x^{n_M} \}$, the procedure is:

\begin{enumerate}[]
    \item For $m = 1,\ldots, M$ (in parallel):\\
          Sample from the subposterior $p_m$, where
            \begin{equation}
                p_m(\theta) \propto p(\theta)^{\frac{1}{M}}
                p(x^{n_m} | \theta). \label{eq:subposterior}
            \end{equation}
    \item Combine the subposterior samples to produce samples from an estimate
        of the subposterior density product $p_1 \cldots p_M$, which is
        proportional to the full-data posterior, i.e. $p_1 \cldots p_M(\theta)
        \propto p(\theta|x^N)$. 
\end{enumerate}

We want to emphasize that we do not need to iterate over these steps and the
combination stage (step 3) is the only step that requires communication between
machines. Also note that sampling from each subposterior (step~2) can typically
be done in the same way as one would sample from the full-data posterior. For
example, when using the Metropolis-Hastings algorithm, one would compute the
likelihood ratio as $\frac{p(\theta^*)^{\frac{1}{M}}p(x^{n_m}|\theta^*)}
{p(\theta)^{\frac{1}{M}}p(x^{n_m}|\theta)}$ instead of
$\frac{p(\theta^*)p(x^{N}|\theta^*)}{p(\theta)p(x^{N}|\theta)}$, where
$\theta^*$ is the proposed move.
In the next section, we show how the combination stage (step 3) is carried out
to generate samples from the full-data posterior using the subposterior
samples.

\section{Combining Subposterior Samples}
Our general idea is to combine the subposterior samples in such a way that we
are implicitly sampling from an estimate of the subposterior density product
function $\widehat{p_1 \cldots p_M}(\theta)$. If our density product estimator
is consistent, then we can show that we are drawing asymptotically exact
samples from the full posterior. 
Further, by studying the estimator error rate, we can explicitly analyze how
quickly the distribution from which we are drawing samples is converging to the
true posterior (and thus compare different combination algorithms).

In the following three sections we present procedures that yield samples from
different estimates of the density product. Our first example is based on a
simple parametric estimator motivated by the Bernstein-von Mises
theorem~\cite{le1986asymptotic}; this procedure generates approximate
(asymptotically biased) samples from the full posterior. Our second
example is based on a nonparametric estimator, and produces asymptotically
exact samples from the full posterior. Our third example is based on a
semiparametric estimator, which combines beneficial aspects from the previous
two estimators while also generating asymptotically exact samples.

\subsection{Approximate posterior sampling with a parametric
estimator}
The first method for forming samples from the full posterior given subposterior
samples involves using an approximation based on the Bernstein-von Mises
(Bayesian central limit) theorem, an important result in Bayesian asymptotic
theory. Assuming that a unique, true data-generating model exists and is
denoted $\theta_0$, this theorem states that the posterior tends to a normal
distribution concentrated around $\theta_0$ as the number of observations
grows. In particular, under suitable regularity conditions, the posterior
$P(\theta|x^N)$ is well approximated by $\mathcal{N}_d(\theta_0,F_N^{-1})$ (where
$F_N$ is the fisher information of the data) when $N$ is large
\cite{le1986asymptotic}. Since we aim to perform posterior sampling when the
number of observations is large, a normal parametric form often serves as a
good posterior approximation. A similar approximation was used in
\cite{ahn2012bayesian} in order to facilitate fast, approximately correct
sampling. We therefore estimate each subposterior density with
$\widehat{p}_m(\theta) = \mathcal{N}_d(\theta | \widehat{\mu}_m,
\widehat{\Sigma}_m)$ where $\widehat{\mu}_m$ and $\widehat{\Sigma}_m$ are the
sample mean and covariance, respectively, of the subposterior samples.  The
product of the $M$ subposterior densities will be proportional to a Gaussian
pdf, and our estimate of the density product function $p_1 \cldots p_M (\theta)
\propto p(\theta | x^N)$ is
\begin{align*}
    \widehat{p_1 \cldots p_M} (\theta) = \widehat{p}_1 \cldots \widehat{p}_M
    (\theta) \propto \mathcal{N}_d \left( \theta | \widehat{\mu}_M,
    \widehat{\Sigma}_M \right),
\end{align*}
where the parameters of this distribution are 
\begin{align}
    &\widehat{\Sigma}_M = \left(\sum_{m=1}^M
    \widehat{\Sigma}_m^{-1} \right)^{-1}
    \label{parametricparams}\\
    &\widehat{\mu}_M = \widehat{\Sigma}_M \left( \sum_{m=1}^M
        \widehat{\Sigma}_m^{-1} \widehat{\mu}_m \right).
        \label{parametricparams2}
\end{align}
These parameters can be computed quickly and, if desired, online (as new
subposterior samples arrive).

\subsection{Asymptotically exact posterior sampling with
nonparametric density product estimation}
\label{sec:nonparametricCombine}
In the previous method we made a parametric assumption based on the
Bernstein-von Mises theorem, which allows us to generate approximate samples
from the full posterior. Although this parametric estimate has quick
convergence, it generates asymptotically biased samples, especially in cases
where the posterior is particularly non-Gaussian. In this section, we develop a
procedure that implicitly samples from the product of nonparametric density
estimates, which allows us to produce asymptotically exact samples from the
full posterior. By constructing a consistent density product estimator from
which we can generate samples, we ensure that the distribution from which we
are sampling converges to the full posterior.

Given $T$ samples\footnote{For ease of description, we assume each machine
generates the same number of samples, $T$. In practice, they do not have to be
the same.}
$\{\theta_{t_m}^m\}_{t_m=1}^{T}$ from a subposterior $p_m$, we can write
the kernel density estimator $\widehat{p}_m(\theta)$ as,
\begin{align*}
    \widehat{p}_m(\theta) 
    &= \frac{1}{T} \sum_{t_m=1}^{T} \frac{1}{h^d} 
    K\left( \frac{\|\theta - \theta_{t_m}^m\|}{h} \right) \nonumber \\
    &= \frac{1}{T} \sum_{t_m=1}^{T} \mathcal{N}_d(\theta|\theta_{t_m}^m,h^2I_d),
\end{align*}
where we have used a Gaussian kernel with bandwidth parameter
$h$. After we have obtained the kernel density estimator
$\widehat{p}_m(\theta)$ for $M$ subposteriors, we define our
nonparametric density product estimator for the full posterior as
\begin{align}
    \widehat{p_1 \cldots p_M}(\theta)
    &= \widehat{p}_1 \cldots \widehat{p}_M (\theta) \nonumber \\
    &= \frac{1}{T^M} \prod_{m=1}^M \sum_{t_m=1}^{T} \mathcal{N}_d(\theta|\theta_{t_m}^m,h^2I_d) \nonumber \\
    &\propto \sum_{t_1=1}^{T} \cldots \sum_{t_M=1}^{T}
    w_{t\cdot}\hspace{2pt}\mathcal{N}_d\left( \theta \Big| \bar{\theta}_{t\cdot}, \frac{h^2}{M} I_d \right).
\end{align}
This estimate is the probability density function (pdf) of a mixture of $T^M$
Gaussians with {\it unnormalized} mixture weights $w_{t\cdot}$. Here, we use
$t\cdot = \{t_1,\ldots,t_M\}$ to denote the set of indices for the $M$ samples
$\{\theta_{t_1}^1,\ldots,\theta_{t_M}^M\}$ (each from a separate machine)
associated with a given mixture component, and we define
\begin{align}
    \bar{\theta}_{t\cdot} &= \frac{1}{M}\sum_{m=1}^M \theta_{t_m}^m \label{thetaBar}\\
    w_{t\cdot} &= \prod_{m=1}^M \mathcal{N}_d\left(\theta_{t_m}^m|
        \bar{\theta}_{t\cdot}, h^2I_d \right). \label{wWeight}
\end{align}
Although there are $T^M$ possible mixture components, we can efficiently
generate samples from this mixture by first sampling a mixture component (based
on its unnormalized component weight $w_{t\cdot}$) and then sampling from this
(Gaussian) component.  In order to sample mixture components, we use an
independent Metropolis within Gibbs (IMG) sampler. This is a form of MCMC,
where at each step in the Markov chain, a single dimension of the current state
is proposed (i.e. sampled) independently of its current value (while keeping
the other dimensions fixed) and then is accepted or rejected.  In our case, at
each step, a new mixture component is proposed by redrawing one of the $M$
current sample indices $t_m \in t\cdot$ associated with the component uniformly
and then accepting or rejecting the resulting proposed component based on its
mixture weight. We give the IMG algorithm for combining subposterior samples
in Algorithm~\ref{alg:nonIMG}.\footnote{Again for simplicity, we assume that we
generate $T$ samples to represent the full posterior, where $T$ is the number
of subposterior samples from each machine.}

In certain situations, Algorithm~\ref{alg:nonIMG} may have a low acceptance
rate and therefore may mix slowly. One way to remedy this is to perform the IMG
combination algorithm multiple times, by first applying it to groups of
$\tilde{M} < M$ subposteriors and then applying the algorithm again to the
output samples from each initial application. For example, one could begin by
applying the algorithm to all $\frac{M}{2}$ pairs (leaving one subposterior
alone if $M$ is odd), then repeating this process---forming pairs and applying
the combination algorithm to pairs only---until there is only one set of
samples remaining, which are samples from the density product estimate.

\begin{algorithm}[h]
    \caption{Asymptotically Exact Sampling via Nonparametric Density Product Estimation}
    \label{alg:nonIMG}
    \begin{algorithmic}[1]
        \REQUIRE Subposterior samples:
        $\{\theta_{t_1}^1\}_{t_1=1}^T\sim p_1(\theta),$ $\ldots,$ 
        $\{\theta_{t_M}^M\}_{t_M=1}^T\sim p_M(\theta)$\vspace{1pt}
        \ENSURE Posterior samples (asymptotically, as
        $T \rightarrow \infty$): $\{\theta_i\}_{i=1}^T\sim p_1 \cldots p_M (\theta)$
                $\propto p(\theta|x^N)$
        \STATE Draw $t\cdot$ $=$ $\{t_1,\ldots,t_M\} \stackrel{\text{iid}}{\sim} \text{Unif}(\{1,\ldots,T\})$
        \FOR{$i=1$ {\bfseries to} $T$}
            \STATE Set $h \leftarrow i^{-1/(4+d)}$
            \FOR{$m=1$ {\bfseries to} $M$}
                \STATE Set $c\cdot \leftarrow t\cdot$
                \STATE Draw $c_m \sim \text{Unif}(\{1,\ldots,T\})$
                \STATE Draw $u \sim \text{Unif}([0,1])$
                \IF{$u < w_{c\cdot}/w_{t\cdot}$} \vspace{1pt}
                    \STATE Set $t\cdot \leftarrow c\cdot$
                \ENDIF
            \ENDFOR
            \STATE Draw $\theta_i \sim \mathcal{N}_d(\bar{\theta}_{t\cdot},\frac{h^2}{M}I_d)$
        \ENDFOR
    \end{algorithmic}
\end{algorithm}

\subsection{Asymptotically exact posterior sampling with
semiparametric density product estimation}
Our first example made use of a parametric estimator, which has
quick convergence, but may be asymptotically biased, while our
second example made use of a nonparametric estimator, which is
asymptotically exact, but may converge slowly when the number of
dimensions is large.  In this example, we implicitly sample from
a semiparametric density product estimate, which allows us to
leverage the fact that the full posterior has a near-Gaussian
form when the number of observations is large, while still
providing an asymptotically unbiased estimate of the posterior
density, as the number of subposterior samples $T
\rightarrow \infty$.

We make use of a semiparametric density estimator for $p_m$ that consists of
the product of a parametric estimator $\widehat{f}_m(\theta)$ $(\text{in our
case } \mathcal{N}_d(\theta|\widehat{\mu}_m,\widehat{\Sigma}_m)\text{ as above})$
and a nonparametric estimator $\widehat{r}(\theta)$ of the correction function
$r(\theta)$ $=$ $p_m(\theta) / \widehat{f}_m(\theta)$
\cite{hjort1995nonparametric}. This estimator gives a near-Gaussian estimate
when the number of samples is small, and converges to the true density as the
number of samples grows. Given $T$ samples $\{\theta_{t_m}^m\}_{t_m=1}^{T}$ from a
subposterior $p_m$, we can write the estimator~as 
\begin{align*}
    \widehat{p}_m(\theta) 
    &= \widehat{f}_m(\theta) 
        \hspace{1pt} \widehat{r}(\theta) \\
    &= \frac{1}{T} \sum_{t_m=1}^{T} \frac{1}{h^d} 
    K\left( \frac{\|\theta - \theta_{t_m}^m\|}{h} \right)
    \frac{\widehat{f}_m(\theta)}{\widehat{f}_m(\theta_{t_m}^m)} \nonumber \\
    &= \frac{1}{T} \sum_{t_m=1}^{T} 
    \frac{\mathcal{N}_d(\theta|\theta_{t_m}^m,h^2I_d)
            \mathcal{N}_d(\theta|\widehat{\mu}_m,\widehat{\Sigma}_m)}
            {\mathcal{N}_d(\theta_{t_m}^m|\widehat{\mu}_m,\widehat{\Sigma}_m)},
\end{align*}
where we have used a Gaussian kernel with bandwidth parameter $h$ for the
nonparametric component of this estimator. Therefore,
we define our semiparametric density product estimator to be
\begin{align*}
    \widehat{p_1 \cldots p_M}(\theta) 
    &= \widehat{p}_1 \cldots \widehat{p}_M (\theta) \nonumber \\
    &= \frac{1}{T^M} \prod_{m=1}^M \sum_{t_m=1}^{T}
        \frac{ \mathcal{N}_d(\theta|\theta_{t_m}^m,hI_d)
        \mathcal{N}_d(\theta|\widehat{\mu}_m,\widehat{\Sigma}_m)}
        {h^d \mathcal{N}_d(\theta_{t_m}^m|\widehat{\mu}_m,\widehat{\Sigma}_m)} \nonumber \\
    &\propto \sum_{t_1=1}^{T} \cldots \sum_{t_M=1}^{T}
        W_{t\cdot}\hspace{2pt}\mathcal{N}_d\left( \theta | \mu_{t\cdot}, \Sigma_{t\cdot} \right).
\end{align*}
This estimate is proportional to the pdf of a mixture of $T^M$ Gaussians with
unnormalized mixture weights,
\begin{align*}
    W_{t\cdot} &= 
        \frac{w_{t\cdot} \hspace{1mm} \mathcal{N}_d\left( \bar{\theta}_{t\cdot}
        | \widehat{\mu}_M, \widehat{\Sigma}_M + \frac{h}{M}I_d \right)}
        {\prod_{m=1}^M
        \mathcal{N}_d(\theta_{t_m}^m|\widehat{\mu}_m,\widehat{\Sigma}_m)},
\end{align*}
where $\bar{\theta}_{t\cdot}$ and $w_{t\cdot}$ are given in Eqs.~\ref{thetaBar}
and \ref{wWeight}. We can write the parameters of a given mixture component
$\mathcal{N}_d(\theta | \mu_{t\cdot}, \Sigma_{t\cdot})$~as
\begin{align*}
    \Sigma_{t\cdot} &= \left(\frac{M}{h}I_d +
    \widehat{\Sigma}_M^{-1}\right)^{-1}, \\
    \mu_{t\cdot} &= \Sigma_{t\cdot} \left(\frac{M}{h}I_d \bar{\theta}_{t\cdot}
        + \widehat{\Sigma}_M^{-1} \widehat{\mu}_M\right),
\end{align*}
where $\widehat{\mu}_M$ and $\widehat{\Sigma}_M$ are given by Eq.
\ref{parametricparams} and \ref{parametricparams2}. We can sample from this
semiparametric estimate using the IMG procedure outlined in
Algorithm~\ref{alg:nonIMG}, replacing the component weights $w_{t\cdot}$ with
$W_{t\cdot}$ and the component parameters $\bar{\theta}_{t\cdot}$ and
$\frac{h}{M}I_d$ with $\mu_{t\cdot}$ and $\Sigma_{t\cdot}$.

We also have a second semiparametric procedure that may give higher acceptance
rates in the IMG algorithm. We follow the above semiparametric procedure, where
each component is a normal distribution with parameters $\mu_{t\cdot}$ and
$\Sigma_{t\cdot}$, but we use the nonparametric component weights $w_{t\cdot}$
instead of $W_{t\cdot}$. This procedure is also asymptotically exact, since the
semiparametric component parameters $\mu_{t\cdot}$ and $\Sigma_{t\cdot}$
approach the nonparametric component parameters $\bar{\theta}_{t\cdot}$ and
$\frac{h}{M}I_d$ as $h \rightarrow 0$, and thus this procedure tends to the
nonparametric procedure given in Algorithm~\ref{alg:nonIMG}.

\section{Method Complexity}
Given $M$ data subsets, to produce $T$ samples in $d$ dimensions with the
nonparametric or semiparametric asymptotically exact procedures (Algorithm~1)
requires $O(dTM^2)$ operations.  The variation on this algorithm that performs
this procedure $M$$-$$1$ times on pairs of subposteriors (to increase the
acceptance rate; detailed in Section~\ref{sec:nonparametricCombine}) instead
requires only $O(dTM)$ operations.

We have presented our method as a two step procedure, where first parallel MCMC
is run to completion, and then the combination algorithm is applied to the $M$
sets of samples. We can instead perform an online version of our algorithm: as
each machine generates a sample, it immediately sends it to a master machine,
which combines the incoming samples\footnote{For the semiparametric method,
this will involve an online update of mean and variance Gaussian parameters.}
and performs the accept or reject step (Algorithm~1, lines 3-12).  This allows
the parallel MCMC phase and the combination phase to be performed in parallel,
and does not require transfering large volumes of data, as only a single sample
is ever transferred at a time.

The total communication required by our method is transferring $O(dTM)$ scalars
($T$ samples from each of $M$ machines), and as stated above, this can be done
online as MCMC is being carried out. Further, the communication is
unidirectional, and each machine does not pause and wait for any information
from other machines during the parallel sampling procedure.

\section{Theoretical Results}
\label{sec:theory}
Our second and third procedures aim to draw asymptotically exact samples by
sampling from (fully or partially) nonparametric estimates of the
density product. We prove the asymptotic correctness of our estimators, and
bound their rate of convergence. This will ensure that we are generating
asymptotically correct samples from the full posterior as the number of samples
$T$ from each subposterior grows.

\subsection{Density product estimate convergence and risk analysis}
To prove (mean-square) consistency of our estimator, we give a bound on the
mean-squared error (MSE), and show that it tends to zero as we increase the
number of samples drawn from each subposterior. To prove this, we first bound
the bias and variance of the estimator. The following proofs make use of
similar bounds on the bias and variance of the nonparametric and
semiparametric density estimators, and therefore the theory applies to both the
nonparametric and semiparametric density product estimators.

Throughout this analysis, we assume that we have $T$ samples $\{\theta_{t_m}^m
\}_{t_m = 1}^T \subset$ $\mathcal{X}$ $\subset$ $\mathbb{R}^d$ from each
subposterior ($m=1,\ldots,M$), and that $h \in \mathbb{R}_+$ denotes the
bandwidth of the nonparametric density product estimator (which is annealed to
zero as $T \rightarrow \infty$ in Algorithm~\ref{alg:nonIMG}).
Let H\"{o}lder class $\Sigma(\beta,L)$ on $\mathcal{X}$ be defined as the set of
all $\ell = \lfloor \beta \rfloor$ times differentiable functions
$f:\mathcal{X} \rightarrow \mathbb{R}$ whose derivative $f^{(l)}$ satisfies
\begin{equation*}
    \lvert f^{(\ell)}(\theta) - f^{(\ell)}(\theta') \rvert
    \leq L \left\lvert \theta-\theta' \right\rvert^{\beta-\ell} \hspace{3mm}
    \text{ for all } \hspace{1mm} \theta,\theta' \in \mathcal{X}.
\end{equation*}
We also define the class of densities $\mathcal{P}(\beta,L)$ to be 
\begin{equation*}
    \mathcal{P}(\beta,L) = \left\{  p \in \Sigma(\beta,L) \hspace{1mm}\Big|\hspace{1mm} p \geq 0,
    \int p(\theta) d\theta=1 \right\}.
\end{equation*}
We also assume that all subposterior densities $p_m$ are bounded, i.e. that
there exists some $b>0$ such that $p_m(\theta) \leq b$ for all
$\theta \in \mathbb{R}^d$ and $m \in \{1,\ldots,M\}$.

First, we bound the bias of our estimator. This shows that the bias tends to
zero as the bandwidth shrinks.
\begin{lemma}
    The bias of the estimator $\widehat{p_1 \cldots p_M}(\theta)$ satisfies 
    \begin{equation*}
        \sup_{\raisebox{-2pt}{$\scriptstyle p_1,\ldots,p_M \in \mathcal{P}(\beta, L)$}}
    \hspace{-5mm} \left| \mathbb{E} \left[ \widehat{p_1 \cldots p_M}(\theta) \right] - p_1 \cldots p_M (\theta)  \right|
        \leq \sum_{m=1}^M c_m h^{m\beta}
    \end{equation*}
    for some $c_1,\ldots,c_M > 0$.
\end{lemma}
\begin{proof}
    For all $p_1,\ldots,p_M \in \mathcal{P}(\beta, L)$, 
    \begin{align*}
        | \mathbb{E} \left[ \widehat{p_1 \cldots p_M} \right] - p_1 \cldots p_M | 
        &= \left| \mathbb{E} \left[ \widehat{p}_1 \cldots \widehat{p}_M \right] - p_1 \cldots p_M \right|\\
        &= \left| \mathbb{E}\left[ \widehat{p_1} \right] \cldots \mathbb{E}\left[ \widehat{p}_M \right] 
            - p_1 \cldots p_M  \right| \\
        &\leq \left| (p_1 + \tilde{c}_1 h^\beta)\cldots(p_M + \tilde{c}_M h^\beta) - p_1 \cldots p_M \right| \\
        &\leq \left| c_1 h^\beta + \ldots + c_M h^{M\beta} \right| \\
        &\leq \left| c_1 h^\beta \right| + \ldots + \left| c_M h^{M\beta} \right| \\
        &= \sum_{m=1}^M c_m h^{m\beta}
    \end{align*}
    where we have used the fact that $ \left| \mathbb{E}\left[ \widehat{p_m}
    \right] - p \right| \leq  \tilde{c}_m h^\beta$ for some~$\tilde{c}_m > 0$.
\end{proof}
Next, we bound the variance of our estimator. This shows that the variance
tends to zero as the number of samples grows large and the bandwidth shrinks.
\begin{lemma}
    The variance of the estimator $\widehat{p_1 \cldots p_M}(\theta)$ satisfies 
    \begin{equation*}
        \sup_{p_1,\ldots,p_M \in \mathcal{P}(\beta, L)}
        \mathbb{V} \left[ \widehat{p_1 \cldots p_M} (\theta) \right]
        \leq \sum_{m=1}^M \binom{M}{m} \frac{c_m}{T^m h^{dm}}
    \end{equation*}
    for some $c_1,\ldots,c_M > 0$ and $0 < h \leq 1$.
\end{lemma}
\begin{proof}
    For all $p_1,\ldots,p_M \in \mathcal{P}(\beta, L)$, 
    \begin{align*}
        \mathbb{V} [\widehat{p_1 \cldots p_M}]
        &= \mathbb{E}\left[ \widehat{p}_1^2 \right] \cldots \mathbb{E}\left[ \widehat{p}_M^2 \right] - 
            \mathbb{E}\left[ \widehat{p}_1 \right]^2 \cldots \mathbb{E}\left[ \widehat{p}_M\right]^2\\
        &= \left( \prod_{m=1}^M  \mathbb{V}\left[ \widehat{p}_m \right] + \mathbb{E}\left[ \widehat{p}_m\right]^2 \right)
            - \left( \prod_{m=1}^M  \mathbb{E}\left[ \widehat{p}_m\right]^2  \right) \\
        &\leq \sum_{m=0}^{M-1} \binom{M}{m} \frac{\tilde{c}^m c^{M-m}}{T^{M-m} h^{d(M-m)}}\\
        &\leq \sum_{m=1}^M \binom{M}{m} \frac{c_m}{T^{m} h^{dm}}
    \end{align*}
    where we have used the facts that $\mathbb{V}\left[ \widehat{p}_m \right]
    \leq \frac{c}{Th^d}$ for some $c>0$ and $\mathbb{E}\left[
    \widehat{p}_m\right]^2 \leq \tilde{c}$ for some $\tilde{c}>0$.
\end{proof}
Finally, we use the bias and variance bounds to bound the MSE, which shows that our estimator is consistent.
\begin{theorem}
    If $h \asymp T^{-1/(2\beta+d)}$, the mean-squared error of the estimator
    $\widehat{p_1 \cldots p_M}(\theta)$ satisfies 
    \begin{align*}
        \sup_{p_1,\ldots,p_M \in \mathcal{P}(\beta, L)}
        \mathbb{E}\left[ \int \left( \widehat{p_1 \cldots p_M}(\theta) - p_1
        \cldots p_M (\theta) \right)^2 d\theta \right]
        \leq \frac{c}{T^{2 \beta / (2 \beta + d)}}
    \end{align*}
    for some $c > 0$ and $0 < h \leq 1$.
\end{theorem}
\begin{proof}
    For all $p_1,\ldots,p_M \in \mathcal{P}(\beta, L)$, using the fact that the
    mean-squared error is equal to the variance plus the bias squared, we have
    that 
    \begin{align*}
        \mathbb{E}\left[ \int \left( \widehat{p_1 \cldots p_M}(\theta) - p_1
                \cldots p_M (\theta) \right)^2 d\theta \right]
            &\leq  \left(\sum_{m=1}^M c_m h^{m\beta} \right)^2 + \sum_{m=1}^M \binom{M}{m} \frac{\tilde{c}_m}{T^m h^{dm}}\\
            &\leq k T^{-2\beta/(2\beta + d)} + \frac{\tilde{k}}{T^{1-d(2\beta+d)}} \hspace{2mm}\text{ (for some $k,\tilde{k}>0$)} \\
            &\leq \frac{c}{T^{2 \beta / (2 \beta + d)}}
    \end{align*}
    for some $c_1,\ldots,c_M > 0$ and $\tilde{c}_1,\ldots,\tilde{c}_M > 0$.
\end{proof}






\section{Method Scope}
The theoretical results and algorithms in this paper hold for posterior
distributions over finite-dimensional real spaces. These include generalized
linear models (e.g. linear, logistic, or Poisson regression), mixture models
with known weights, hierarchical models, and (more generally)
finite-dimensional graphical models with unconstrained variables. This also
includes both unimodal and multimodal posterior densities (such as in
Section~\ref{sec:gmmSec}). However, the methods and theory presented here do
not yet extend to cases such as infinite dimensional models (e.g. nonparametric
Bayesian models \cite{gershman2012tutorial}) nor to distributions over the
simplex (e.g. topics in latent Dirichlet allocation \cite{blei2003latent}). In
the future, we hope to extend this work to these domains.

\section{Related Work}
In~\cite{welling2011bayesian,ahn2012bayesian,Patterson:2013}, the authors
develop a way to sample approximately from a posterior distribution when only a
small randomized mini-batch of data is used at each step.
In~\cite{korattikara2013austerity}, the authors used a hypothesis test to
decide whether to accept or reject proposals using a small set of data
(adaptively) as opposed to the exact Metropolis-Hastings rule. This reduces the
amount of time required to compute the acceptance ratio. Since all of these
algorithms are still sequential, they can be directly used in our algorithm to
generate subposterior samples to further speed up the entire sampling process.

Several parallel MCMC algorithms have been designed for specific models, such
as for topic models~\cite{smola2010architecture,Newman:2009} and nonparametric
mixture models~\cite{williamson2013parallel}. These approaches still require
synchronization to be correct (or approximately correct), while ours aims for
more general model settings and does not need synchronization until the final
combination stage.

Consensus Monte Carlo~\cite{bayesAndBigData} is perhaps the most relevant work
to ours. In this algorithm, data is also portioned into different machines and
MCMC is performed independently on each machine. Thus, it roughly has the same
time complexity as our algorithm. However, the prior is not explicitly
reweighted during sampling as we do in Eq~\ref{eq:subposterior}, and final
samples for the full posterior are generated by averaging subposterior samples.
Furthermore, this algorithm has few theoretical guarantees.  We find that this
algorithm can be viewed as a relaxation of our nonparametric, asymptotically
exact sampling procedure, where samples are generated from an evenly weighted
mixture (instead of each component having weight $w_{t\cdot}$) and where each
sample is set to $\bar{\theta}_{t\cdot}$ instead of being drawn from
$\mathcal{N} \left( \bar{\theta}_{t\cdot}, \frac{h}{M} I_d \right)$. This
algorithm is one of our experimental baselines.

\section{Empirical Study}
In the following sections, we demonstrate empirically that our
method allows for quicker, MCMC-based estimation of a posterior
distribution, and that our consistent-estimator-based procedures
yield asymptotically exact results.  We show our method on a few
Bayesian models using both synthetic and real data. In each
experiment, we compare the following strategies for parallel,
communication-free sampling:\footnote{We did not directly compare
with the algorithms that require synchronization since the setup
of these experiments can be rather different. We plan to explore
these comparisons in the extended version of this paper.}
\begin{itemize*}
    \item \textbf{Single chain full-data posterior samples}
        (\texttt{regularChain})---Typical, single-chain MCMC for sampling from
        the full-data posterior.
    \item \textbf{Parametric subposterior density product estimate}
        (\texttt{parametric})---For $M$ sets of subposterior samples, the
        combination yielding samples from the parametric density product
        estimate.
    \item \textbf{Nonparametric subposterior density product estimate}
        (\texttt{nonparametric})---For $M$ sets of subposterior samples, the
        combination yielding samples from the nonparametric density product
        estimate. 
    \item \textbf{Semiparametric subposterior density product estimate}
        (\texttt{semiparametric})---For $M$ sets of subposterior samples, the
        combination yielding samples from the semiparametric density product
        estimate.
    \item \textbf{Subposterior sample average} (\texttt{subpostAvg})---For $M$
        sets of subposterior samples, the average of $M$ samples consisting of
        one sample taken from  each subposterior.
    \item \textbf{Subposterior sample pooling} (\texttt{subpostPool})---For $M$
        sets of subposterior samples, the union of all sets of samples.
    \item \textbf{Duplicate chains full-data posterior sample pooling}
        (\texttt{duplicateChainsPool})---For $M$ sets of samples from the
        full-data posterior, the union of all sets of samples.
\end{itemize*}

To assess the performance of our sampling and combination strategies, we ran a
single chain of MCMC on the full data for 500,000 iterations, removed the first
half as burn-in, and considered the remaining samples the ``groundtruth''
samples for the true posterior density.  We then needed a general method to
compare the distance between two densities given samples from each, which holds
for general densities (including multimodal densities, where it is ineffective
to compare moments such as the mean and variance\footnote{In these cases,
dissimilar densities might have similar low-order moments. See
Section~\ref{sec:gmmSec} for an example.}).  Following work in density-based
regression \cite{oliva2013distribution}, we use an estimate of the $L_2$
distance, $d_2(p,\hat{p})$, between the groundtruth posterior density $p$ and a
proposed posterior density $\widehat{p}$, where $d_2(p,\hat{p}) = \|p -
\widehat{p} \|_2 = \left(\int (p(\theta) - \widehat{p}(\theta))^2 d\theta
\right)^{1/2}$.

In the following experiments involving timing, to compute the posterior $L_2$
error at each time point, we collected all samples generated before a given
number of seconds, and added the time taken to transfer the samples and combine
them using one of the proposed methods. In all experiments and methods, we
followed a fixed rule of removing the first $\frac{1}{6}$ samples for burn-in
(which, in the case of combination procedures, was applied to each set of
subposterior samples before the combination was performed).

Experiments were conducted with a standard cluster system. We obtained
subposterior samples by submitting batch jobs to each worker since these jobs
are all independent. We then saved the results to the disk of each worker and
transferred them to the same machine which performed the final combination.

\begin{figure}[t]
        \center{\includegraphics[width=0.95\textwidth]{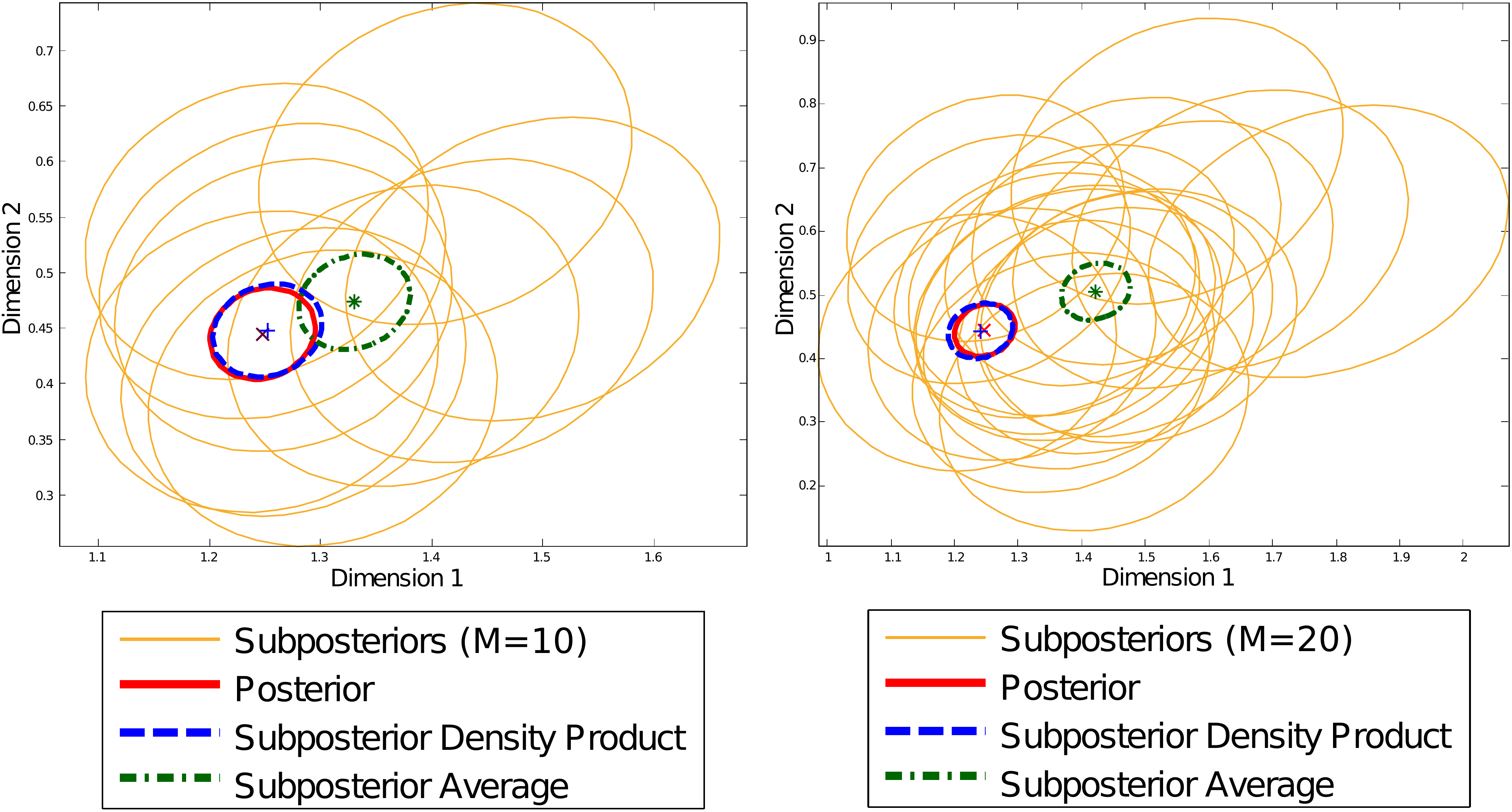}}
        \caption{Bayesian logistic regression posterior ovals. We show the
            posterior $90\%$ probability mass ovals for the first 2-dimensional
            marginal of the posterior, the $M$ subposteriors, the subposterior
            density product (via the \texttt{parametric} procedure), and the
            subposterior average (via the \texttt{subpostAvg} procedure). We
            show $M$=10 subsets (left) and $M$=20 subsets (right). The
            subposterior density product generates samples that are consistent
            with the true posterior, while the \texttt{subpostAvg} produces
            biased results, which grow in error as $M$ increases.
        }
        \label{fig:blr-ovals}
\end{figure}

\subsection{Generalized Linear Models}
Generalized linear models are widely used for many regression and
classification problems. Here we conduct experiments, using logistic regression
as a test case, on both synthetic and real data to demonstrate the speed of our
parallel MCMC algorithm compared with typical MCMC strategies. 

\subsubsection{Synthetic data}
Our synthetic dataset contains 50,000 observations in $50$ dimensions. To
generate the data, we drew each element of the model parameter $\beta$ and data
matrix $X$ from a standard normal distribution, and then drew each outcome as
$y_i \sim \text{Bernoulli}(\text{logit}^{-1}(X_i \beta))$ (where $X_i$ denotes
the $i^{th}$ row of $X$)\footnote{Note that we did not explicitly include the intercept
term in our logistic regression model.}. We use Stan, an automated Hamiltonian
Monte Carlo (HMC) software package,\footnote{http://mc-stan.org} to perform
sampling for both the true posterior (for groundtruth and comparison methods)
and for the subposteriors on each machine. One advantage of Stan is that it is
implemented with C++ and uses the No-U-Turn sampler for HMC, which does not
require any user-provided parameters~\cite{hoffman2011no}.   

In Figure~\ref{fig:blr-ovals}, we illustrate results for logistic regression,
showing the subposterior densities, the subposterior density product, the
subposterior sample average, and the true posterior density, for the number of
subsets $M$ set to 10 (left) and 20 (right). Samples generated by our approach
(where we draw samples from the subposterior density product via the
\texttt{parametric} procedure) overlap with the true posterior much better than
those generated via the \texttt{subpostAvg} (subposterior sample average)
procedure--- averaging of samples appears to create systematic biases. Futher,
the error in averaging appears to increase as $M$ grows. In
Figure~\ref{fig:blr-curves} (left) we show the posterior error vs time. A
regular full-data chain takes much longer to converge to low error compared
with our combination methods, and simple averaging and pooling of subposterior
samples gives biased solutions.


We next compare our combination methods with multiple independent ``duplicate''
chains each run on the full dataset.  Even though our methods only require a
fraction of the data storage on each machine, we are still able to achieve a
significant speed-up over the full-data chains. This is primarily because the
duplicate chains cannot parallelize burn-in (i.e. each chain must still take
some $n$ steps before generating reasonable samples, and the time taken to
reach these $n$ steps does not decrease as more machines are added). However,
in our method, each subposterior sampler can take each step more quickly,
effectively allowing us to decrease the time needed for burn-in as we increase
$M$. We show this empirically in Figure~\ref{fig:blr-curves} (right), where we
plot the posterior error vs time, and compare with full duplicate chains as $M$
is increased.

Using a Matlab implementation of our combination algorithms, all (batch)
combination procedures take under twenty seconds to complete on a 2.5GHz Intel
Core i5 with 16GB memory.

\begin{figure}[t]
        \center{\includegraphics[width=0.95\textwidth]{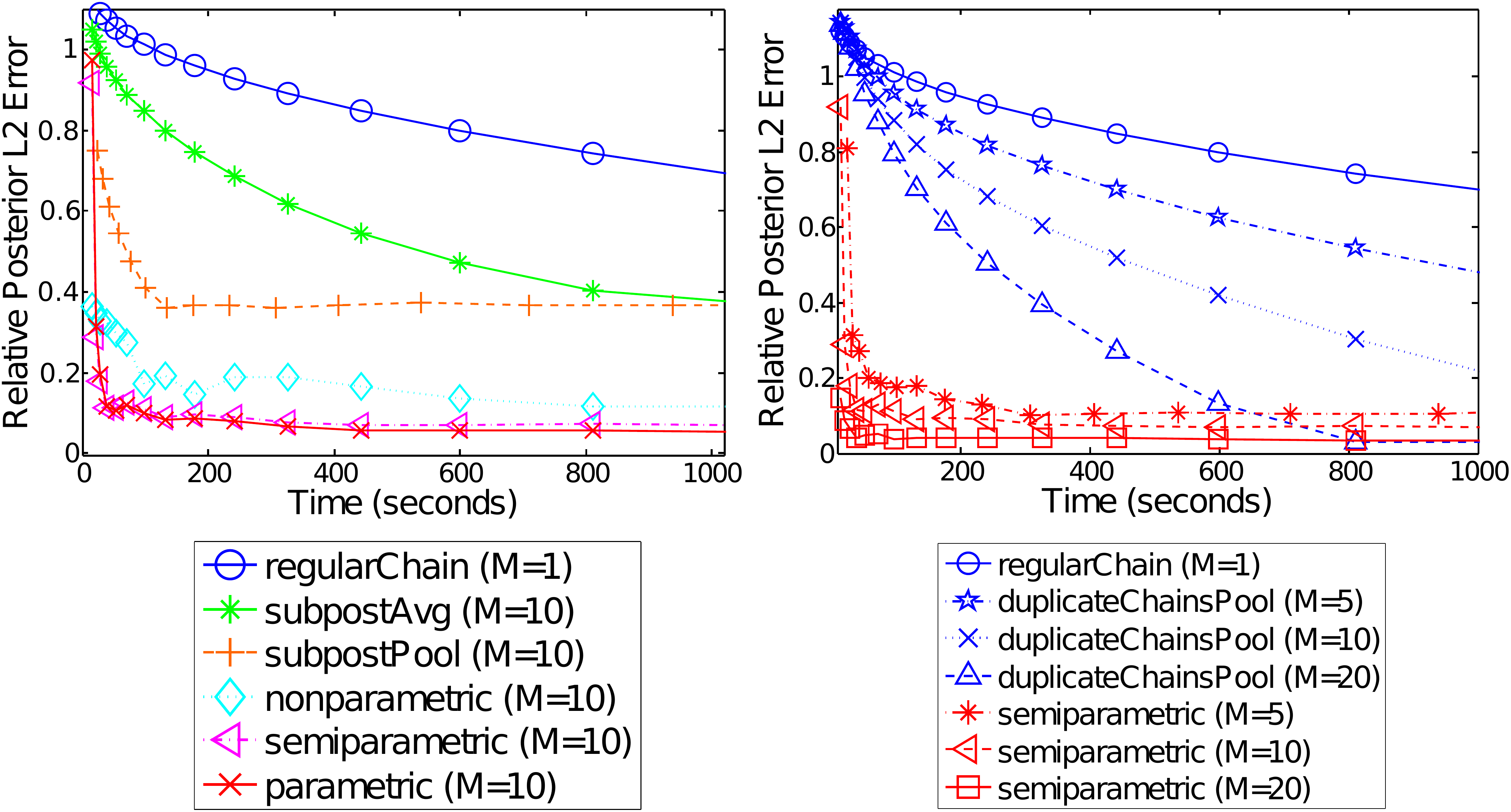}}
        \caption{Posterior $L_2$ error vs time for logistic regression. Left:
            the three combination strategies proposed in this paper
            (\texttt{parametric}, \texttt{nonparametric}, and
            \texttt{semiparametric}) reduce the posterior error much more
            quickly than a single full-data Markov chain; the
            \texttt{subpostAvg} and \texttt{subpostPool} procedures yield
            biased results.  Right: we compare with multiple full-data Markov
            chains (\texttt{duplicateChainsPool}); our method yields faster
            convergence to the posterior even though only a fraction of the
            data is being used by each chain.
        }
        \label{fig:blr-curves}
\end{figure}



\subsubsection{Real-world data}
Here, we use the \emph{covtype} (predicting forest cover
types)\footnote{http://www.csie.ntu.edu.tw/\~{}cjlin/libsvmtools/datasets}
dataset, containing 581,012 observations in $54$ dimensions. A single chain
of HMC running on this entire dataset takes an average of 15.76 minutes per
sample; hence, it is infeasible to generate groundtruth samples for this
dataset. Instead we show classification accuracy vs time. For a given
set of samples, we perform classification using a sample estimate of the
posterior predictive distribution for a new label $y$ with associated datapoint
$x$, i.e.
\begin{align*}
    P(y|x,y^N,x^N) &= \int P(y|x,\beta,y^N,x^N)P(\beta|x^N,y^N) \\
                   &\approx \frac{1}{S}\sum_{s=1}^S P(y|x,\beta_s)
\end{align*}
where $x^N$ and $y^N$ denote the $N$ observations, and $P(y|x,\beta_s) =
\text{Bernoulli}(\text{logit}^{-1}(x^\top \beta_s))$. Figure~\ref{fig:covtype}
(left) shows the results for this task, where we use $M$=50 splits. The
parallel methods achieve a higher accuracy much faster than the single-chain
MCMC algorithm.

\begin{figure}[t]
        \center{\includegraphics[width=0.95\textwidth]{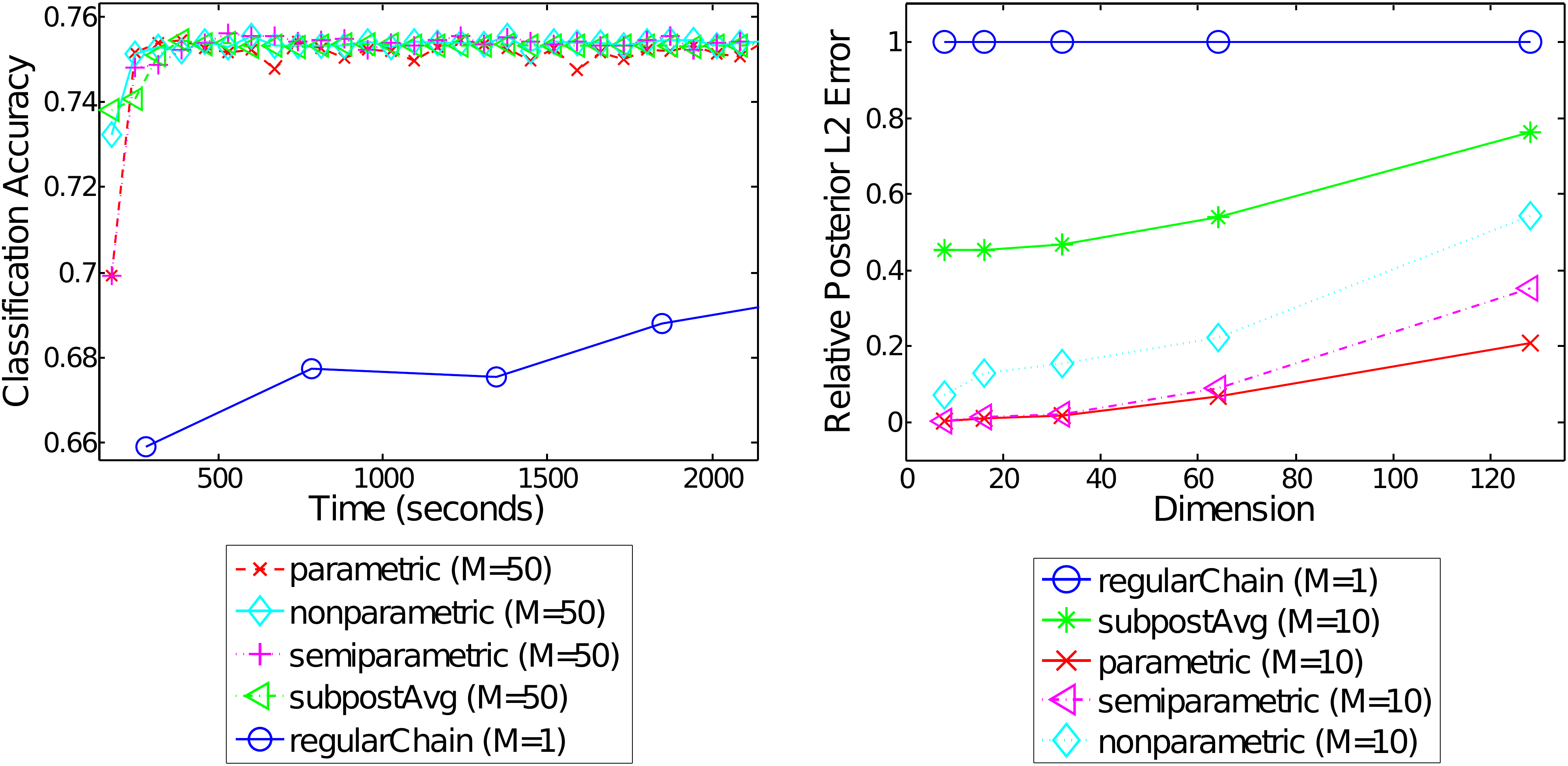}}
        \caption{Left: Bayesian logistic regression classification accuracy vs
            time for the task of predicting forest cover type. Right: Posterior
            error vs dimension on synthetic data at 1000 seconds, normalized so
            that \texttt{regularChain} error is fixed at 1.
    } \label{fig:covtype}
\end{figure}

\subsubsection{Scalability with dimension}
We investigate how the errors of our methods scale with dimensionality, to
compare the different estimators implicit in the combination procedures.  In
Figure~\ref{fig:covtype} (right) we show the relative posterior error (taken at
1000 seconds) vs dimension, for the synthetic data with $M$=10 splits. The
errors at each dimension are normalized so that the \texttt{regularChain} error
is equal to 1.  Here, the \texttt{parametric} (asymptotically biased) procedure
scales best with dimension, and the \texttt{semiparametric} (asymptotically
exact) procedure is a close second.  These results also demonstrate that,
although the \texttt{nonparametric} method can be viewed as implicitly sampling
from a nonparametric density estimate (which is usually restricted to
low-dimensional densities), the performance of our method does not suffer
greatly when we perform parallel MCMC on posterior distributions in much
higher-dimensional spaces. 

\subsection{Gaussian mixture models}
\label{sec:gmmSec}
In this experiment, we aim to show correct posterior sampling in cases where
the full-data posterior, as well as the subposteriors, are multimodal. We will
see that the combination procedures that are asymptotically biased suffer
greatly in these scenarios. To demonstrate this, we perform sampling in a
Gaussian mixture model. We generate 50,000 samples from a ten component mixture
of 2-d Gaussians. The resulting posterior is multimodal; this can be seen by
the fact that the component labels can be arbitrarily permuted and achieve the
same posterior value. For example, we find after sampling that the posterior
distribution over each component mean has ten modes. To sample from this
multimodal posterior, we used the Metropolis-Hastings algorithm, where the
component labels were permuted before each step (note that this permutation
results in a move between two points in the posterior distribution with equal
probability). 

In Figure~\ref{fig:gaussian} we show results for $M$$=$$10$ splits, showing
samples from the true posterior, overlaid samples from all five subposteriors,
results from averaging the subposterior samples, and the results after applying
our three subposterior combination procedures. This figure shows the 2-d
marginal of the posterior corresponding to the posterior over a single mean
component.  The \texttt{subpostAvg} and \texttt{parametric} procedures both
give biased results, and cannot capture the multimodality of the posterior.  We
show the posterior error vs time in Figure~\ref{fig:gmm_hier_curves} (left),
and see that our asymptotically exact methods yield quick convergence to low
posterior error.

\begin{figure}[t]
        \center{\vspace{2mm} \includegraphics[width=0.85\textwidth]{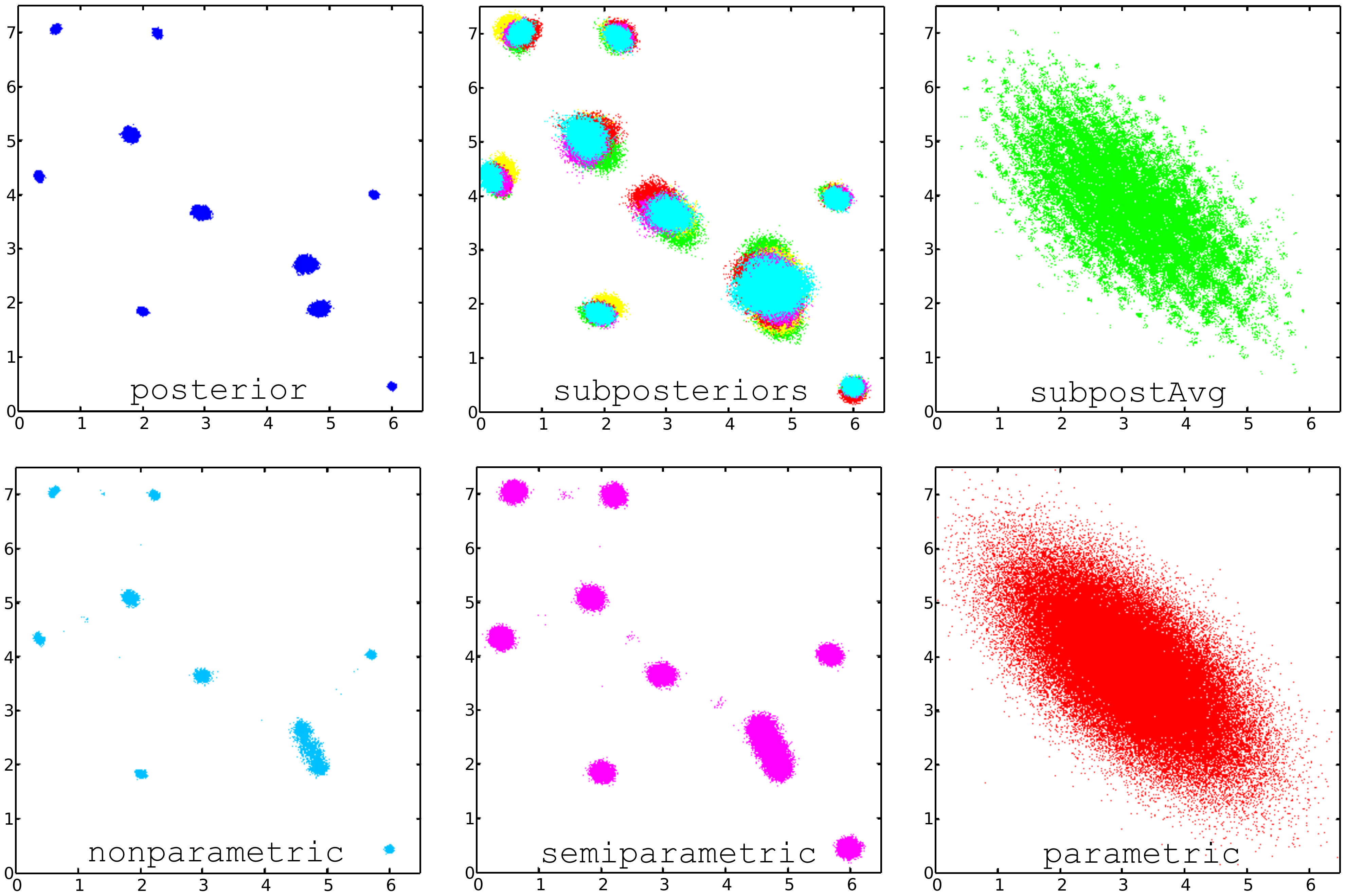}}
        \caption{Gaussian mixture model posterior samples. We show 100,000
            samples from a single 2-d marginal (corresponding to the posterior
            over a single mean parameter) of the full-data posterior (top left), all
            subposteriors (top middle---each one is given a unique color), the
            subposterior average via the \texttt{subpostAvg} procedure (top
            right), and the subposterior density product via the
            \texttt{nonparametric} procedure (bottom left),
            \texttt{semiparametric} procedure (bottom middle), and
            \texttt{parametric} procedure (bottom right).
  }\label{fig:gaussian}
\end{figure}

\subsection{Hierarchical models}
We show results on a hierarchical Poisson-gamma model of the following form
\begin{align*} a &\sim \text{Exponential}(\lambda)\\ b &\sim
\text{Gamma}(\alpha,\beta)\\ q_i &\sim \text{Gamma}(a,b) \text{ for
$i=1,\ldots,N$}\\ x_i &\sim \text{Poisson}(q_i t_i) \text{ for $i=1,\ldots,N$}
\end{align*}
for $N$=50,000 observations. We draw $\{ x_i \}_{i=1}^N$ from the above
generative process (after fixing values for $a$, $b$, $\lambda$, and $\{ t_i
\}_{i=1}^N$), and use $M$=10 splits. We again perform MCMC using the Stan
software package. 

In Figure~\ref{fig:gmm_hier_curves} (right) we show the posterior error vs
time, and see that our combination methods complete burn-in and converge to a
low posterior error very quickly relative to the \texttt{subpostAvg} and
\texttt{subpostPool} procedures and full-data chains.

\begin{figure}[t]
        \center{\includegraphics[width=0.95\textwidth]{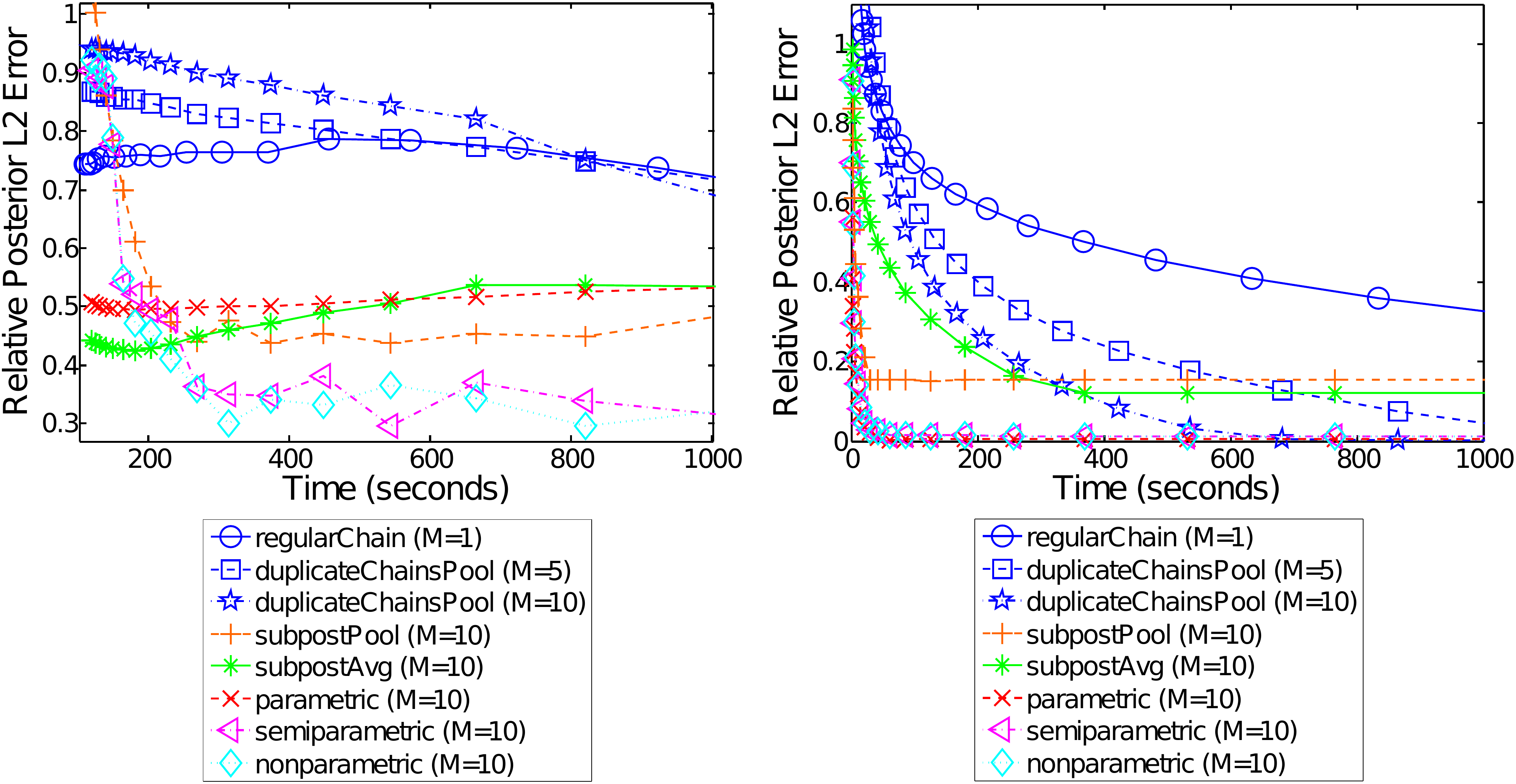}}
        \caption{Left: Gaussian mixture model posterior error vs time results.
            Right: Poisson-gamma hierarchical model posterior error vs time
            results.
    } \label{fig:gmm_hier_curves}
\end{figure}







\section{Discussion and Future Work}

In this paper, we present an embarrassingly parallel MCMC algorithm and provide
theoretical guarantees about the samples it yields. Experimental results
demonstrate our method's potential to speed up burn-in and perform faster
asymptotically correct sampling. Further, it can be used in settings where data
are partitioned onto multiple machines that have little
intercommunication---this is ideal for use in a MapReduce setting.  Currently,
our algorithm works primarily when the posterior samples are real,
unconstrained values and we plan to extend our algorithm to more general
settings in future work.
\newpage



{\small \bibliographystyle{amsplain}
        \bibliography{embarrassinglyParallelMCMC_neiswangerWangXing}}

\providecommand{\bysame}{\leavevmode\hbox to3em{\hrulefill}\thinspace}
\providecommand{\MR}{\relax\ifhmode\unskip\space\fi MR }
\providecommand{\MRhref}[2]{%
  \href{http://www.ams.org/mathscinet-getitem?mr=#1}{#2}
}
\providecommand{\href}[2]{#2}
\begin{thebibliography}{10}

\bibitem{Agarwal:2012}
Alekh Agarwal and John~C Duchi, \emph{Distributed delayed stochastic
  optimization}, Decision and Control (CDC), 2012 IEEE 51st Annual Conference
  on, IEEE, 2012, pp.~5451--5452.

\bibitem{ahn2012bayesian}
Sungjin Ahn, Anoop Korattikara, and Max Welling, \emph{Bayesian posterior
  sampling via stochastic gradient fisher scoring}, Proceedings of the 29th
  International Conference on Machine Learning, 2012, pp.~1591--1598.

\bibitem{blei2003latent}
David~M Blei, Andrew~Y Ng, and Michael~I Jordan, \emph{Latent dirichlet
  allocation}, The Journal of Machine Learning Research \textbf{3} (2003),
  993--1022.

\bibitem{dean2008mapreduce}
Jeffrey Dean and Sanjay Ghemawat, \emph{Mapreduce: simplified data processing
  on large clusters}, Communications of the ACM \textbf{51} (2008), no.~1,
  107--113.

\bibitem{gershman2012tutorial}
Samuel~J Gershman and David~M Blei, \emph{A tutorial on bayesian nonparametric
  models}, Journal of Mathematical Psychology \textbf{56} (2012), no.~1, 1--12.

\bibitem{hjort1995nonparametric}
Nils~Lid Hjort and Ingrid~K Glad, \emph{Nonparametric density estimation with a
  parametric start}, The Annals of Statistics (1995), 882--904.

\bibitem{Ho:2013}
Qirong Ho, James Cipar, Henggang Cui, Seunghak Lee, Jin~Kyu Kim, Phillip~B.
  Gibbons, Gregory~R. Ganger, Garth Gibson, and Eric~P. Xing, \emph{More
  effective distributed ml via a stale synchronous parallel parameter server},
  Advances in Neural Information Processing Systems, 2013.

\bibitem{hoffman2011no}
Matthew~D Hoffman and Andrew Gelman, \emph{The no-u-turn sampler: Adaptively
  setting path lengths in hamiltonian monte carlo}, arXiv preprint
  arXiv:1111.4246 (2011).

\bibitem{korattikara2013austerity}
Anoop Korattikara, Yutian Chen, and Max Welling, \emph{Austerity in {MCMC}
  land: Cutting the {M}etropolis-{H}astings budget}, arXiv preprint
  arXiv:1304.5299 (2013).

\bibitem{Langford:2009}
John Langford, Alex~J Smola, and Martin Zinkevich, \emph{Slow learners are
  fast}, Advances in Neural Information Processing Systems, 2009.

\bibitem{laskey2003population}
Kathryn~Blackmond Laskey and James~W Myers, \emph{Population {M}arkov chain
  {M}onte {C}arlo}, Machine Learning \textbf{50} (2003), no.~1-2, 175--196.

\bibitem{le1986asymptotic}
Lucien Le~Cam, \emph{Asymptotic methods in statistical decision theory}, New
  York (1986).

\bibitem{murray2010distributed}
Lawrence Murray, \emph{Distributed {M}arkov chain {M}onte {C}arlo}, Proceedings
  of Neural Information Processing Systems Workshop on Learning on Cores,
  Clusters and Clouds, vol.~11, 2010.

\bibitem{Newman:2009}
David Newman, Arthur Asuncion, Padhraic Smyth, and Max Welling,
  \emph{Distributed algorithms for topic models}, The Journal of Machine
  Learning Research \textbf{10} (2009), 1801--1828.

\bibitem{oliva2013distribution}
Junier Oliva, Barnab{\'a}s P{\'o}czos, and Jeff Schneider, \emph{Distribution
  to distribution regression}, Proceedings of The 30th International Conference
  on Machine Learning, 2013, pp.~1049--1057.

\bibitem{Patterson:2013}
Sam Patterson and Yee~Whye Teh, \emph{Stochastic gradient riemannian langevin
  dynamics on the probability simplex}, Advances in Neural Information
  Processing Systems, 2013.

\bibitem{bayesAndBigData}
Steven~L. Scott, Alexander~W. Blocker, and Fernando~V. Bonassi, \emph{Bayes and
  big data: The consensus monte carlo algorithm}, Bayes 250, 2013.

\bibitem{smola2010architecture}
Alexander Smola and Shravan Narayanamurthy, \emph{An architecture for parallel
  topic models}, Proceedings of the VLDB Endowment \textbf{3} (2010), no.~1-2,
  703--710.

\bibitem{welling2011bayesian}
Max Welling and Yee~W Teh, \emph{Bayesian learning via stochastic gradient
  {L}angevin dynamics}, Proceedings of the 28th International Conference on
  Machine Learning, 2011, pp.~681--688.

\bibitem{wilkinson2006parallel}
Darren~J Wilkinson, \emph{Parallel {B}ayesian computation}, Statistics
  Textbooks and Monographs \textbf{184} (2006), 477.

\bibitem{williamson2013parallel}
Sinead Williamson, Avinava Dubey, and Eric~P Xing, \emph{Parallel {M}arkov
  chain {M}onte {C}arlo for nonparametric mixture models}, Proceedings of the
  30th International Conference on Machine Learning, 2013, pp.~98--106.

\end{thebibliography}

\end{document}